\documentclass[letterpaper]{article}
\usepackage{proceed2e}
\usepackage[margin=1in]{geometry}

\usepackage{times}

\usepackage{graphicx,amssymb,float,MnSymbol,tikz,amsthm}
\usepackage[round]{natbib}
\usetikzlibrary{arrows,shapes}
\usetikzlibrary{decorations.markings}

\newtheorem{theorem}{Theorem}
\newtheorem{lemma}{Lemma}

\def\ci{\!\perp\!}

\def\ra{\rightarrow}
\def\la{\leftarrow}

\def\no{\multimap}

\newcommand{\comments}[1]{}

\tikzset{tt/.style={decoration={
  markings,
  mark=at position .485 with {\arrow{>}},
  mark=at position .515 with {\arrow{<}}},postaction={decorate}}}

\begin{document}

\title{Identifiability of Gaussian Structural Equation Models with Dependent Errors Having Equal Variances}
\author{ {\bf Jose M. Pe\~{n}a} \\
Department of Computer and Information Science \\
Link{\"o}ping University\\
58183 Link{\"o}ping, Sweden\\
}

\maketitle

\begin{abstract}
In this paper, we prove that some Gaussian structural equation models with dependent errors having equal variances are identifiable from their corresponding Gaussian distributions. Specifically, we prove identifiability for the Gaussian structural equation models that can be represented as Andersson-Madigan-Perlman chain graphs \citep{Anderssonetal.2001}. These chain graphs were originally developed to represent independence models. However, they are also suitable for representing causal models with additive noise \citep{Penna2016}. Our result implies then that these causal models can be identified from observational data alone. Our result generalizes the result by \citet{PetersandBuhlmann2014}, who considered independent errors having equal variances. The suitability of the equal error variances assumption should be assessed on a per domain basis.
\end{abstract}

\section{PRELIMINARIES}

All the graphs and probability distributions in this paper are defined over a finite set $X$. The elements of $X$ are not distinguished from singletons. Uppercase letters denote random variables and lowercase letters denote random variables' values.

The parents of a set of nodes $S$ of a graph $G$ is the set $Pa_G(S) = \{X_j | X_j \ra X_k$ is in $G$ with $X_k \in S \}$. The descendants of $S$ is the set $De_G(S) = \{X_j | X_k \ra \ldots \ra X_j$ is in $G$ with $X_k \in S \}$. The non-descendants of $S$ is the set $ND_G(S)=X \setminus De_G(S)$. The adjacents of $S$ is the set $Ad_G(S) = \{X_j | X_j \ra X_k$, $X_j \la X_k$ or $X_j - X_k$ is in $G$ with $X_k \in S \}$. A route from a node $X_1$ to a node $X_n$ in $G$ is a sequence of (not necessarily distinct) nodes $X_1, \ldots, X_n$ such that $X_j \in Ad_G(X_{j+1})$ for all $1 \leq j < n$. A route is called a cycle if $X_n=X_1$. A cycle is called a semidirected cycle if it is of the form $X_1 \ra X_2 \no \cdots \no X_n$ where $\no$ is a short for $\ra$ or $-$. A chain graph (CG) is a simple graph with directed and/or undirected edges, and without semidirected cycles. A set of nodes of a CG $G$ is connected if there exists a route in $G$ between every pair of nodes in the set and such that all the edges in the route are undirected. A chain component of $G$ is a maximal connected set.

We now recall the interpretation of CGs due to \citet{Anderssonetal.2001}, \citet{Levitzetal.2001} and \citet{Penna2016}, also known as AMP CGs. A node $X_k$ in a route $\rho$ in a CG $G$ is called a triplex node in $\rho$ if $X_j \ra X_k \la X_l$, $X_j \ra X_k - X_l$, or $X_j - X_k \la X_l$ is a subroute of $\rho$. Moreover, $\rho$ is said to be $C$-open with $C \subseteq X$ when (i) every triplex node in $\rho$ is in $C$, and (ii) every non-triplex node in $\rho$ is outside $C$. Let $A$, $B$ and $C$ denote three disjoint subsets of $X$. When there is no $C$-open route in $G$ between a node in $A$ and a node in $B$, we say that $A$ is separated from $B$ given $C$ in $G$ and denote it as $A \ci_G B | C$.\footnote{\citet{Anderssonetal.2001} originally interpreted CGs via the so-called augmentation criterion. \citet[Theorem 4.1]{Levitzetal.2001} introduced the so-called p-separation criterion and proved its equivalence to the augmentation criterion. \citet[Theorem 2]{Penna2016} introduced the route-based criterion that we use in this paper and proved its equivalence to the p-separation criterion.} The statistical independences represented by $G$ are the separations $A \ci_G B | C$. A probability distribution $p$ is Markovian with respect to $G$ if the independences represented by $G$ are a subset of those in $p$. If the two sets of independences coincide, then $p$ is faithful to $G$. If $G$ has an induced subgraph of the form $X_j \ra X_k \la X_l$, $X_j \ra X_k - X_l$ or $X_j - X_k \la X_l$, then we say that $G$ has a triplex $(X_j,X_k,X_l)$. Two CGs are said to be Markov equivalent if the set of distributions that are Markovian with respect to each CG is the same. We know that two CGs are Markov equivalent if and only if they have the same adjacencies and the same triplexes \citep[Theorem 5]{Anderssonetal.2001}. The following lemma gives an additional characterization.

\begin{lemma}\label{lem:eq}
Two CGs $G$ and $H$ are Markov equivalent if and only if they represent the same independences.
\end{lemma}

\begin{proof}
The if part is trivial. To see the only if part, note that \citet[Theorem 6.1]{Levitzetal.2001} prove that there are Gaussian distributions $p$ and $q$ that are faithful to $G$ and $H$, respectively. Moreover, $p$ is Markovian with respect to $H$, because $G$ and $H$ are Markov equivalent. Likewise for $q$ and $G$. Therefore, $G$ and $H$ must represent the same independences.
\end{proof}

\begin{table}
\caption{Algorithm for magnifying an AMP CG.}\label{tab:magnification}
\begin{center}
\begin{tabular}{|rl|}
\hline
&\\
& {\bf Input}: An AMP CG $G_0$.\\
& {\bf Output}: The magnified AMP CG $G'_0$.\\
&\\
1 & Set $G'_0=G_0$\\
2 & For each node $X_j$ in $G_0$\\
3 & \hspace{0.3cm} Add the node $N_j$ and the edge $N_j \ra X_j$ to $G'_0$\\
4 & For each edge $X_j - X_k$ in $G_0$\\
5 & \hspace{0.3cm} Replace $X_j - X_j$ with the edge $N_j - N_j$ in $G'_0$\\
6 & Return $G'_0$\\
&\\
\hline
\end{tabular}
\end{center}
\end{table}

\section{IDENTIFIABILITY}

Consider a structural equation model (SEM) with AMP CG $G_0$. The system includes an equation for each $X_j \in X$, which is of the form
\begin{equation}\label{eq:sem}
X_j = \sum_{X_k \in Pa_{G_0}(X_j)} \beta_{jk} X_k + N_j
\end{equation}
where the set of all the error variables $N_j$, hereinafter denoted by $N$, is distributed according to $\mathcal{N}(0, \Sigma)$ where $\Sigma$ is positive definite and such that $\Sigma^{-1}_{jk} = 0$ for every edge $X_j - X_k$ that is not in $G_0$. This is where our work differs from that by \citet{PetersandBuhlmann2014}, who require that $G_0$ is a directed and acyclic graph (DAG), i.e. $\Sigma^{-1}_{jk} = 0$ for all $j \neq k$. Note that every probability distribution $p$ specified by Equation \ref{eq:sem} is distributed according to $\mathcal{N}(\mu, (I-\beta)^{-1} \Sigma (I-\beta^T)^{-1})$, and is Markovian with respect to $G_0$ \citep[Theorems 10 and 11]{Penna2016}.\footnote{\citet[p. 51]{Anderssonetal.2001} describes an alternative but equivalent specification of the SEM in Equation \ref{eq:sem}. The SEM can also be seen as a generalized seemingly unrelated regression (SUR) model \citep{Zellner1962}, or a joint response graph model with dashed arrows and full lines (also called concentration regression graph model) \citep{CoxandWermuth1996}.}

The error terms $N_j$ are represented implicitly in $G_0$. They can be represented explicitly by magnifying $G_0$ into the AMP CG $G'_0$ as shown in Table \ref{tab:magnification}. The magnification basically consists in adding the error nodes $N_j$ to $G_0$ and connect them appropriately. Figure \ref{fig:example2} shows an example. Note that Equation \ref{eq:sem} implies that $X_j$ is determined by $Pa_{G_0}(X_j) \cup N_j$ and $N_j$ is determined by $X_j \cup Pa_{G_0}(X_j)$. Formally, we say that $W \in X \cup N$ is determined by $Z \subseteq X \cup N$ when $W \in Z$ or $W$ is a function of $Z$. From the point of view of the separations, that a node outside the conditioning set of a separation is determined by the conditioning set has the same effect as if the node were actually in the conditioning set. Bearing this in mind, $G_0$ and $G'_0$ represent the same separations over $X$ \citep[Theorem 9]{Penna2016}.

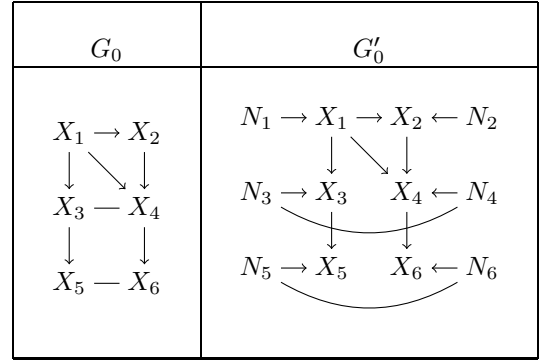
\begin{figure}
\begin{center}
\begin{tabular}{|c|c|}
\hline
&\\
$G_0$&$G'_0$\\
\hline
&\\
\begin{tabular}{c}
\begin{tikzpicture}[inner sep=1mm]
\node at (0,0) (A) {$X_1$};
\node at (1,0) (B) {$X_2$};
\node at (0,-1) (C) {$X_3$};
\node at (1,-1) (D) {$X_4$};
\node at (0,-2) (E) {$X_5$};
\node at (1,-2) (F) {$X_6$};
\path[->] (A) edge (B);
\path[->] (A) edge (C);
\path[->] (A) edge (D);
\path[->] (B) edge (D);
\path[-] (C) edge (D);
\path[->] (C) edge (E);
\path[->] (D) edge (F);
\path[-] (E) edge (F);
\end{tikzpicture}
\end{tabular}
&
\begin{tabular}{c}
\begin{tikzpicture}[inner sep=1mm]
\node at (0,0) (A) {$X_1$};
\node at (1,0) (B) {$X_2$};
\node at (0,-1) (C) {$X_3$};
\node at (1,-1) (D) {$X_4$};
\node at (0,-2) (E) {$X_5$};
\node at (1,-2) (F) {$X_6$};
\node at (-1,0) (EA) {$N_1$};
\node at (2,0) (EB) {$N_2$};
\node at (-1,-1) (EC) {$N_3$};
\node at (2,-1) (ED) {$N_4$};
\node at (-1,-2) (EE) {$N_5$};
\node at (2,-2) (EF) {$N_6$};
\path[->] (EA) edge (A);
\path[->] (EB) edge (B);
\path[->] (EC) edge (C);
\path[->] (ED) edge (D);
\path[->] (EE) edge (E);
\path[->] (EF) edge (F);
\path[->] (A) edge (B);
\path[->] (A) edge (C);
\path[->] (A) edge (D);
\path[->] (B) edge (D);
\path[->] (C) edge (E);
\path[->] (D) edge (F);
\path[-] (EC) edge [bend right] (ED);
\path[-] (EE) edge [bend right] (EF);
\end{tikzpicture}
\end{tabular}\\
&\\
\hline
\end{tabular}
\caption{Example of the magnification of an AMP CG.}\label{fig:example2}
\end{center}
\end{figure}

A less formal but more intuitive interpretation of AMP CGs is as follows. We can interpret the parents of each node in a CG as its observed causes. Its unobserved causes are summarized by an error node that is represented implicitly in the CG. We can interpret the undirected edges in the CG as the dependence relationships between the different error nodes. The causal structure is constrained to be a directed and acyclic graph, but the dependence structure can be any undirected graph. This causal interpretation of AMP CGs parallels that of acyclic directed mixed graphs \citep{Richardson2003,Pearl2009}. However, whereas a missing edge between two error nodes in acyclic directed mixed graphs represents marginal independence, in AMP CGs it represents conditional independence given the rest of the error nodes.

We assume that the coefficients $\beta_{jk}$ in Equation \ref{eq:sem} and the non-zero entries of the concentration matrix $\Sigma^{-1}$ have been selected at random. This implies that $p$ is faithful to $G_0$ with probability almost 1 \cite[Theorem 6.1]{Levitzetal.2001}. We therefore assume faithfulness hereinafter. We also assume that $p$ has been preprocessed by rescaling $N_j$ in Equation \ref{eq:sem} as $N_j \sigma / \sigma_j$ for all $j$, where $\sigma_j^2=\Sigma_{jj}$ and $\sigma^2$ is an arbitrary positive value. This implies that we assume that the errors have equal variance, namely $\sigma^2$. The following lemma proves that, after the rescaling, the error covariance matrix is still positive definite and keeps all the previous (in)dependencies which implies that, after the rescaling, $p$ is still most likely faithful to $G_0$.

\begin{lemma}\label{lem:rescaling}
Consider rescaling $N_j$ in Equation \ref{eq:sem} as $N_j \sigma / \sigma_j$ for all $j$. Then, the error covariance matrix has the same independences before and after the rescaling. Moreover, the error covariance matrix is positive definite after the rescaling if and only if it was so before the rescaling.
\end{lemma}

\begin{proof}
Let $\Sigma$ and $\overline{\Sigma}$ denote the error covariance matrices before and after the rescaling, respectively. Note that we only need to consider independences between singletons. In particular, $N_j$ is independent of $N_k$ conditioned on $N_L$ if and only if $(\Sigma_{jkL})^{-1}_{jk}=0$ \cite[Proposition 5.2]{Lauritzen1996} and, thus, if and only if $det(A)=0$ where $A$ is the result of removing the row $j$ and column $k$ from $\Sigma_{jkL}$. Note that $det(A)=\sum_{\pi \in S} sign(\pi) \prod_i A_{i \pi(i)}$ where $S$ denotes all the permutations over the number of rows or columns of $A$. Then, $det(\overline{A})=det(A) \prod_i \sigma^2/\sigma_i^2$ and, thus, $det(A)=0$ if and only $det(\overline{A})=0$.

A matrix is positive definite if and only if the determinants of all its upper-left submatrices are positive. Therefore, it follows from the previous paragraph that $\Sigma$ and $\overline{\Sigma}$ are both positive definite or none.
\end{proof}

We now present an auxiliary result and then our main result.

\begin{lemma}\label{lem:conditioning}
Let $X$ be distributed according to $\mathcal{N}(\mu, \Sigma)$ with positive definite $\Sigma$. Let $A$ and $B$ denote a non-trivial partition of $X$. Also, let $A^*=A|_b$ in distribution. Then, $var(A^*_j) \leq var(A_j)$ for all $j$ and $b$.
\end{lemma}

\begin{proof}
Note that $var(A^*)=\Sigma_{AA} - \Sigma_{AB} \Sigma_{BB}^{-1} \Sigma_{BA}$, which implies that $var(A^*_j) = \Sigma_{jj} - \Sigma_{jB} \Sigma_{BB}^{-1} \Sigma_{Bj}$, which implies that $var(A^*_j) \leq var(A_j)$ since $\Sigma_{BB}^{-1}$ is positive definite.
\end{proof}

\begin{theorem}\label{the:main}
Let $p$ be a probability distribution generated by a SEM with AMP CG $G_0$ and equal error variances. Then, $G_0$ is identifiable from $p$.
\end{theorem}

\begin{proof}
Under the faithfulness assumption, we can identify the Markov equivalence class of $G_0$ from $p$. An efficient way of doing it is the learning algorithm by \citet{Penna2012}. Any member of the equivalence class can be transformed into any other member by a sequence of feasible splits and mergings \cite[Theorem 3]{SonntagandPenna2015}. These (opposite) operations transform an AMP CG into a Markov equivalent chain graph by splitting or merging two chain components of the former. Therefore, we can assume to the contrary that there are two SEMs that induce $p$ and such that the CG $H$ corresponding to the latter is the result of a feasible merge in the CG $G$ corresponding to the former. As the name suggest, a feasible merge of two chain components $U$ and $L$ of $G$ implies dropping the direction of the edges between $U$ and $L$. Therefore, $Pa_G(U \cup L) \setminus U = Pa_{H}(U \cup L)$ and $ND_G(U \cup L) = ND_{H}(U \cup L)$. Let $Q=Pa_G(U \cup L) \setminus U$, and note that
\[
U \cup L \ci_G ND_G(U \cup L) \setminus Q | Q
\]
and
\begin{equation}\label{eq:ind2}
U \cup L \ci_{H} ND_{H}(U \cup L) \setminus Q | Q.
\end{equation}
Note from $G$ that $L= \beta_L [Q, U] + N_L$ by Equation \ref{eq:sem}. Note also that $N_L \ci_{G'} Q \cup U$ where $G'$ is the magnification of $G$. Let $L^*=L|_q$ and $U^*=U|_q$ in distribution. Then
\[
L^*=f(q) + \beta^*_L U^* + N_L
\]
in distribution for some linear function $f$ \cite[Lemma A2]{Petersetal.2014}. Then
\[
var(L^*)= \beta^*_L var(U^*) (\beta^*_L)^T + var(N_L)
\]
because $N_L \ci_{G'} U^*$. Note that for some $X_j \in L$, we have that $\beta^*_j$ is non-zero and, thus, $\beta^*_j var(U^*) (\beta^*_j)^T > 0$ because $var(U^*)$ is positive definite. Then, $var(X_j) > \sigma^2$. However, we also have that $var(X_j) \leq \sigma^2$ by Lemma \ref{lem:conditioning}. Therefore, we have reached a contradiction and thus $G=H$.
\end{proof}

It follows from the theorem above that two AMP CGs that represent the same independences are not Markov equivalent under the constraint of equal error variances, i.e. Lemma \ref{lem:eq} does not hold under this constraint.

\section{GREEDY SEARCH}

\citet{DrtonandEichler2006} describe an iterative procedure for maximum likelihood estimation of the SEM parameters associated with an AMP CG, i.e. the coefficients $\beta_{jk}$ in Equation \ref{eq:sem} and the non-zero entries of the concentration matrix $\Sigma^{-1}$. They also show that the procedure is consistent. The procedure estimates the parameters for each chain component separately. For a given component, it alternates between estimating the regression coefficients and estimating the error covariance matrix. The former step consists in a generalized least squares formula. The latter step consists in running the iterative proportional fitting procedure on the regression residuals \citep{Lauritzen1996,WainwrightandJordan2008}.\footnote{\citet{DrtonandEichler2006} mention that the first step of their procedure consists in a generalized least squares (GLS) formula. Their formula certainly resembles the classical one. Therefore, one may wonder whether their procedure coincides with iterated feasible GLS estimation. In other words, consider the multivariate regression model $Y=\beta X + \epsilon$ where $\epsilon$ is distributed as $\mathcal{N}(0, \Sigma)$. If $\Sigma=I \sigma^2$, then the ordinary least squares (OLS) estimates of $\beta$ can be computed as $\hat{\beta}=(X^T X)^{-1} X^T Y$. These are also the maximum likelihood (ML) estimates of $\beta$. If $\Sigma \neq I \sigma^2$ then the GLS estimates of $\beta$ can be computed as $\hat{\beta}=(X^T \Sigma^{-1} X)^{-1} X^T \Sigma^{-1} Y$. Again, these are also the ML estimates of $\beta$, because GLS estimation can be seen as a data transformation so that OLS estimation applies. When $\Sigma$ is unknown (as in our case), it can be estimated, e.g. by running the iterative proportional fitting procedure on the residuals of OLS estimation. This is known as feasible GLS estimation. One can use the resulting estimates of $\beta$ to recompute the residuals and thus the estimate of $\Sigma$. This is known as iterated feasible GLS estimation. It is unknown to us if this procedure results in consistent ML estimates.}

In the future, we would like to extend the procedure of \citet{DrtonandEichler2006} by incorporating the equal error variances constraint in the iterative proportional fitting procedure. A similar extension to incorporate parameter equality constraints in the concentration or correlation matrix of undirected Gaussian graphical models has been proposed by \citet{HojsgaardandLauritzen2008}. Such an extension may allow us to develop a penalized maximum likelihood score for AMP CGs that is consistent by Theorem \ref{the:main}, i.e. $G_0$ asymptotically maximizes the score. Such a score may allow us to develop the following learning algorithms. We can use the PC-like learning algorithm developed by \citet{Penna2012} for AMP CGs in order to identify the Markov equivalence class of $G_0$ and, then, use the penalized maximum likelihood score to orient the edges as in $G_0$. However, the first step assumes faithfulness. We can alternatively perform a greedy search in the space of AMP CGs guided by the penalized maximum likelihood score. This approaches does not assume faithfulness at the cost of higher computational cost and the risk of getting trapped in a equivalence class that does not include $G_0$. In the future, we would like to implement the latter algorithm and evaluate it when the equal error variances assumption holds and does not hold.

\comments{
\section{Discussion}

Generalizable to CAMs, i.e. $X_j = \sum_{X_k \in Pa_{G_0}(A)} \beta_{jk} f(X_k) + N_j$ ? Can these be seen as SVMs and thus as a convex problem ?
}

\bibliographystyle{plainnat}
\bibliography{IdentifiabilityAMPCGs}

\end{document}